\documentclass[letterpaper]{article} 
\usepackage{aaai25}  
\usepackage{times}  
\usepackage{helvet}  
\usepackage{courier}  
\usepackage[hyphens]{url}  
\usepackage{graphicx} 
\usepackage{natbib}  
\usepackage{caption} 
\usepackage{pgfplots}
\usepgfplotslibrary{groupplots}
\pgfplotsset{compat=newest}
\usepackage{comment}
\usepackage{amsmath}
\usepackage{amsthm}
\usepackage{amsfonts}
\usepackage{centernot}
\usepackage{tikz}
\newtheorem{remark}{Remark}
\newtheorem{definition}{Definition}

\usepackage{amsmath,amssymb,amsfonts}
\usepackage{algorithmic}
\usepackage{algorithm} 
\usepackage{graphicx}
\usepackage{natbib}
\usepackage{booktabs}
\usepackage{siunitx}
\usepackage{thmtools} 
\usepackage{thm-restate}
\usepackage{lipsum}
\usepackage{multirow}
\usepackage{mathtools}
\usepackage{graphicx}
\usepackage{booktabs}
\usepackage{adjustbox}
\usepackage{array}
\usepackage{pgfplots}
\pgfplotsset{compat=newest}
\usepackage{float}
\usepackage[capitalise]{cleveref}

\usepackage{soul}
\usepackage{xcolor}
\newcommand{\w}[1]{\textcolor{blue}{#1}}

\usepackage{amssymb}
\usepackage{pifont}
\usepackage{fancyhdr}
\DeclareMathOperator*{\argmax}{argmax}
\definecolor{darkblue}{rgb}{0, 0.2, 0.4}

\usepackage{algorithm}
\usepackage{algorithmic}

\usepackage{newfloat}
\usepackage{listings}
\DeclareCaptionStyle{ruled}{labelfont=normalfont,labelsep=colon,strut=off} 
\floatstyle{ruled}
\newfloat{listing}{tb}{lst}{}
\floatname{listing}{Listing}
\title{DeepSN: A Sheaf Neural Framework for Influence Maximization}
\author{
    Asela Hevapathige,
    Qing Wang,
    Ahad N. Zehmakan
}
\affiliations{
    Graph Research Lab, School of Computing, Australian National University
    \\
    \{asela.hevapathige, qing.wang, ahadn.zehmakan\}@anu.edu.au
}

\usepackage{bibentry}

\begin{document}

\maketitle

\begin{abstract}
Influence maximization is key topic in data mining, with broad applications in social network analysis and viral marketing. In recent years, researchers have increasingly turned to machine learning techniques to address this problem. They have developed methods to learn the underlying diffusion processes in a data-driven manner, which enhances the generalizability of the solution, and have designed optimization objectives to identify the optimal seed set. Nonetheless, two fundamental gaps remain unsolved: (1) Graph Neural Networks (GNNs) are increasingly used to learn diffusion models, but in their traditional form, they often fail to capture the complex dynamics of influence diffusion, (2) Designing optimization objectives is challenging due to combinatorial explosion when solving this problem. To address these challenges, we propose a novel framework, DeepSN. Our framework employs sheaf neural diffusion to learn diverse influence patterns in a data-driven, end-to-end manner, providing enhanced separability in capturing diffusion characteristics. We also propose an optimization technique that accounts for overlapping influence between vertices, which helps to reduce the search space and identify the optimal seed set effectively and efficiently. Finally, we conduct extensive experiments on both synthetic and real-world datasets to demonstrate the effectiveness of our framework.
\end{abstract}

%

\section{Introduction}

Influence maximization (IM) is a challenging network science problem that involves identifying a set of vertices which, when activated, maximize the spread of influence across the network. IM has significant real-world applications including  viral marketing \cite{li2009discovering,kempe2003maximizing}, disease control \cite{marquetoux2016using}, social media content management \cite{hosseini2017community}, and crisis communication \cite{fan2021role}. Despite decades of research, IM still remains challenging primarily due its exponentially large search space and the complex nature of the influence diffusion processes. 

A plethora of traditional methods have been proposed to obtain optimal or near-optimal solutions for IM \cite{kempe2003maximizing,leskovec2007cost,wang2010community,tang2015influence,li2019tiptop}. These methods vary in strategy and effectiveness: some offer theoretical performance guarantees, while others use heuristics to enhance scalability. A common trait among these traditional approaches is their reliance on the explicit specification of the influence diffusion model as input. Recently, researchers have turned to learning-based methods focusing on their ability to automatically identify the diffusion model from the ground truth and accommodate multiple diffusion models, thereby achieving broader applicability \cite{li2023survey}.

Learning-based approaches for IM encounter two key challenges. \emph{(1) Effectively modeling the underlying diffusion model from ground truth: } There is a growing interest in leveraging Graph Neural Networks (GNNs) for modeling influence propagation, due to their ability to capture structural insights in networks~\cite{kumar2022influence,ling2023deep,panagopoulos2023maximizing}. However, existing work relies on traditional GNN techniques like attention \cite{lee2019attention} and convolution \cite{zhang2019graph}. While these traditional GNNs excel in tasks with a static nature such as vertex classification and graph regression, their inductive bias and inherent assumptions limit their ability to model the dynamic behaviors of influence diffusion.  They also suffer from issues like over-smoothing \cite{chen2020measuring} which hampers their capacity to capture global information and long-range dependencies crucial for influence diffusion \cite{xia2021deepis}. \emph{(2) Identifying the optimal subset of vertices with maximal Influence: } Designing an optimization objective to select the optimal seed set is arduous due to the vast search space, especially in large networks. Existing learning-based approaches often use deep reinforcement learning \cite{li2022piano,chen2023touplegdd} and ranking-based methods \cite{kumar2022influence,panagopoulos2020influence} to approximate the optimal seed set. However, these methods often incur significant computational costs and suffer from a lack of interpretability.


Our work addresses the limitations of existing approaches by leveraging sheaf theory \cite{tennison1975sheaf,bredon2012sheaf}, an algebraic-topological framework that captures the topological and geometric properties of complex networks, essential for understanding dynamic processes. Current sheaf diffusion GNNs \cite{hansen2020sheaf,bodnar2022neural,barbero2022sheaf} are limited by their homogeneous diffusion behaviors and inability to handle the evolving dynamics crucial for influence propagation. To tackle this, we propose \emph{DeepSN}, a novel sheaf GNN framework designed to address the influence maximization problem. Our architecture effectively estimates diverse influence propagation processes by learning adaptive structural relationships between vertices, enabling effective identification of influence patterns in the network. Building on this, we design a seed set inference mechanism that uses subgraphs to efficiently reduce the search space in influence maximization. Our contributions are summarized as follows. 

\begin{itemize}
  \item \textbf{Influence Diffusion: } We redefine influence propagation as a sheaf diffusion-reaction process, capturing the intricate dynamics found in real-world diffusion models.

  \item \textbf{GNN Architecture: } We propose a novel GNN architecture rooted in sheaf theory, designed to capture the unique nuances of influence propagation.
  \item \textbf{Seed Selection:} We use a subgraph-based strategy leveraging the structural insights from our sheaf GNN to minimize overlapping influence among seed vertices, thereby reducing the search space for influence maximization.
  \item \textbf{Experiments: } We rigorously evaluate our framework on both real-world and synthetic datasets, demonstrating superior performance and generalizability over state-of-the-art methods.
  
\end{itemize}

\section{Related Work}

\paragraph{Influence Maximization Methods} IM methods generally fall into two main categories: traditional and learning-based. Traditional approaches include simulation-based, proxy-based, and sketch-based methods \cite{li2018influence}. Simulation-based methods rely on Monte Carlo simulations for stochastic evaluation with theoretical guarantees, while proxy-based methods use heuristics for efficient seed set approximation. Sketch-based methods combine both simulation and proxy methods, balancing theoretical guarantees and computational efficiency. For detailed discussions on these traditional methods, see the surveys by \citet{li2018influence} and \citet{banerjee2020survey}.

Contrary to traditional methods that require a specific diffusion model as input, learning-based models excel in generalizability, adapting to multiple diffusion models. Several studies have leveraged deep reinforcement learning for this problem~\cite{li2022piano, ma2022influence, wang2021reinforcement, chen2023touplegdd}, and some works explored the use of GNNs in this context~\cite{xia2021deepis,kumar2022influence,ling2023deep,panagopoulos2023maximizing}.However, reinforcement learning methods often face scalability challenges due to exploration complexity, limiting their application in large-scale networks. Meanwhile, GNNs typically rely on traditional architectures that struggle to capture dynamic diffusion phenomena and mainly focus on progressive models, limiting their adaptability to non-progressive scenarios. Our work breaks from these limitations by incorporating an inductive bias guided by diffusion-reaction processes, which is capable of modeling both information diffusion and intrinsic transformations that occur during the diffusion process. This allows us to effectively model complex propagation patterns in both progressive and non-progressive spread dynamics. 

\paragraph{Diffusion GNNs} 
A diffusion process typically refers to the spread of information across a structure over time. In graph representation learning, diffusion manifests in various areas, including graph generation \cite{liu2023generative} and information propagation \cite{khoshraftar2024survey}. Generative diffusion GNNs \cite{niu2020permutation,bao2022equivariant} use diffusion processes to learn graph distributions and generate new graphs through iterative, learned operations. On the other hand, information propagation-based diffusion GNNs \cite{gasteiger2019diffusion,chamberlain2021grand,zhao2021adaptive}  model the propagation of information in graphs by discretizing an underlying partial differential equation. Sheaf theory was integrated into diffusion GNNs by \citet{bodnar2022neural}. Subsequently, several studies have adopted sheaf-based GNNs for a variety of downstream tasks \cite{duta2024sheaf,caralt2024complex,nguyen2024sheaf}. Reaction-diffusion models improve diffusion by adding regularization. Several diffusion-based GNNs \cite{wang2022acmp, choi2023gread, eliasof2024feature} incorporate reaction terms as constraints to prevent oversmoothing and preserve the distinctiveness of vertex features.

Our work fundamentally differs from existing diffusion models. While generative diffusion models focus on global graph generation, we model influence propagation as a dynamic process within graphs. Traditional diffusion GNNs often treat propagation as a uniform process, focused on influence spread, overlooking individual vertex dynamics and neighboring transitions. In contrast, we incorporate reaction terms into sheaf structures to capture evolving dynamics and vertex transitions, setting our method apart from models that primarily target static tasks.






 \section{Problem Formulation}


Let $G=(V, E)$ be a graph with the vertex set $V$,  the edge set $E$, $|V| = n$ and  $|E|=m$. We denote $A \in \{0, 1\}^{n \times n}$  to be the adjacency matrix of $G$. The set of neighboring vertices for vertex $v$ is denoted as $N(v) = \{u \in V \mid (u,v)\in E\}$.

\paragraph{Influence Maximization} Influence maximization is a graph optimization problem that aims to identify a subset of vertices, known as the \emph{seed set} (i.e., initially activated vertices), in a graph to maximize influence spread according to a specified diffusion function. We formally define it below.

\begin{definition}[Influence Maximization Problem]
Given a graph \( G = (V, E) \), the \emph{influence maximization} (IM) problem is to find a subset \( S^{\star} \subseteq V \) of up to \( k \) vertices that maximizes an expected influence diffusion function \( \sigma \). More precisely,
\[
S^{\star} = \argmax_{|S| \leq k} \sigma(S, G; \theta).
\]
Here, \( S^{\star} \) is referred to as the optimal seed set and \( \sigma(S, G; \theta) \) represents the expected influence diffusion of the seed set \( S \) in the graph \( G \) (that is, the expected final number of activations) with model parameters \( \theta \). 
\end{definition}

The influence diffusion function \(\sigma(\cdot; \theta) \) 
can be instantiated with various models to capture different dynamics of influence spread. In the Linear Threshold (LT) model \cite{granovetter1978threshold}, \( \theta \) denotes vertex thresholds and edge weights, with vertices becoming active when the weighted sum of neighbors' influences exceeds their thresholds. In the Independent Cascade (IC) model \cite{goldenberg2001using}, \( \theta \) represents edge activation probabilities, where active vertices have a single chance to activate each neighbor. The Susceptible-Infected-Susceptible (SIS) model \cite{d2008note} extends this by incorporating infection and recovery rates in \( \theta \), allowing vertices to transition between susceptible and infected states over time. 

\paragraph{Our Work} In this paper, we introduce a deep learning framework called \emph{DeepSN} to address the influence maximization problem. Our framework comprises two phases: \emph{learning to estimate influence} and \emph{optimizing seed selection}. In the influence estimation phase, our goal is to measure the total influence exerted by a set of initially activated vertices in a learnable way. This estimated influence is then used in the seed selection optimization phase to identify the optimal seed set that maximizes overall influence. \cref{fig:DeepSNN summary} provides an overview of our framework.

\vspace{0.15cm}
 Note that all eliminated proofs from the main content, due to space constraints, are provided in the appendix.


\begin{figure*}[]
  \centering  \includegraphics[width=1\textwidth, height=0.18\textheight]{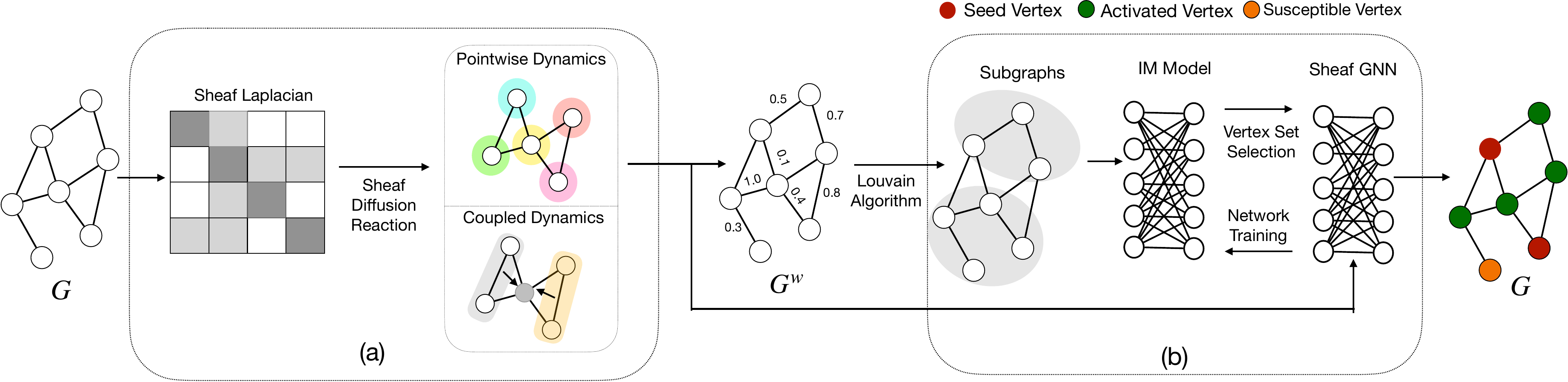}
  \caption{The DeepSN framework consists of two phases: (a) learning to estimate influence with sheaf GNN; b) optimizing seed selection using the subgraphs, IM model, and trained sheaf GNN.}
  \label{fig:DeepSNN summary}
\end{figure*}

\section{Learning to Estimate Influence}

We introduce a novel sheaf GNN model based on sheaf theory, specifically designed to adapt to the complex diffusion patterns inherent in influence propagation. 


\subsection{Topological Sheaf Diffusion}
We begin by introducing the concept of \emph{cellular sheaf}, the foundational building block of our GNN.

\begin{definition} [Cellular Sheaf]
     A \emph{(cellular) sheaf} $(G, F)$ on a graph $G=(V,E)$ consists of  a vector space  $F_v$ for each vertex $v \in V$ ,  a vector space $F_e$ for each edge $e \in E$, a linear map $\mathcal{F}_{v \trianglelefteq e}: \mathcal{F}_v \to \mathcal{F}_e$ for each incident vertex-edge pair $v \trianglelefteq e$. 
\end{definition}

 \( \mathcal{F}_v \) and \( \mathcal{F}_e \) are the \emph{vertex sheaf} and \emph{edge sheaf}, respectively, while \( \mathcal{F}_{v \trianglelefteq e} \) is  the \emph{transformation map}. Let $\bigoplus$ denote the direct sum of vector spaces. The space of \emph{0-cochains} 
is the direct sum of vector spaces over the vertices: 
$C^{0} (G, F) = \bigoplus_{v \in V} F_v$, and the space of \emph{1-cochains} 
is the direct sum of vector spaces over the edges: 
$C^{1} (G, F) = \bigoplus_{e \in E} F_e$. 

 We formulate influence propagation in a network as an opinion propagation process, where each vertex corresponds to a vertex sheaf, representing a private opinion. Edge sheaves represent public opinions related to influence propagation, while transformation maps translate information from vertex sheaves to edge sheaves, extracting public opinions from private opinions. Let \( x_v \in F_v \) be a \( d \)-dimensional feature vector for each \( v \in V \). The \emph{coboundary map} \( \delta: C^0(G, F) \rightarrow C^1(G, F) \) is a linear transformation,
defined for an edge $(v, u)$ as 
\begin{equation}
 \delta(x)_e = \mathcal{F}_{v \trianglelefteq e} . x_v - \mathcal{F}_{u \trianglelefteq e} . x_u
\end{equation}
where \( \mathcal{F}_{v \trianglelefteq e} \cdot x_v \) and \( \mathcal{F}_{u \trianglelefteq e} \cdot x_u \) represent the public opinions associated with \( e \), originating from vertices \( v \) and \( u \), respectively.


The sheaf Laplacian of a vertex measures the disparity between its public opinion and that of its neighbors. To capture the fine-grained nuances of trust, a crucial factor in modeling influence propagation, we incorporate learnable sheaf coefficients to model each vertex's confidence in the opinions received from its neighbors.

\begin{definition} [Non-linear Sheaf Laplacian]
Let $\psi_{vu} \in [0, 1]$. The sheaf Laplacian $L_\mathcal{F} : C^0(G; F) \to C^0(G; F)$ on a sheaf $(G, F)$ is defined vertex-wise as
\begin{equation}
L_\mathcal{F}(x_v) = \sum_{v,u \trianglelefteq e} \mathcal{F}_{v \trianglelefteq e}^T . \psi_{vu}. \left(\mathcal{F}_{v \trianglelefteq e} . x_v - \mathcal{F}_{u \trianglelefteq e} . x_u \right).
\label{eq:sl}
\end{equation}
 \end{definition}

Here, \( \psi_{vu} \) is the sheaf coefficient which is learnable. Note that the sheaf Laplacian employed by \citet{bodnar2022neural} is a specific instance of our Laplacian, where \( \psi_{u,v} = 1 \) for all \( (u,v) \in E \).


 To ensure that the Laplacian matrix $\hat{L}_\mathcal{F}$ remains positive definite, a necessary condition for the convergence of the diffusion process, we apply the following modification:

\begin{equation}
   \hat{L}_\mathcal{F} = L_\mathcal{F} + \epsilon \cdot I  
\end{equation}
 where $\epsilon$ is a scalar, $I$ is the identity matrix, and $\cdot$ denotes element-wise multiplication. The following lemma establishes the necessary and sufficient conditions for $\hat{L}_\mathcal{F}$.

 \begin{restatable}[]{lemma}{lemmps} Let \(\lambda_{\text{min}}\) denote the smallest eigenvalue of \(L_\mathcal{F}\).
\(\hat{L}_\mathcal{F}\) is positive definite if and only if $\epsilon > -\lambda_{\text{min}}$. 
\end{restatable}


The sheaf diffusion operator (i.e., normalized sheaf Laplacian) models opinion diffusion by capturing discrepancies between vertex opinions and enabling smooth information propagation across the network, defined as
\begin{equation}
    \Delta_{\mathcal{F}} = D^{-\frac{1}{2}} L_\mathcal{F} D^{-\frac{1}{2}},
\end{equation}
where \( D \) is the block-diagonal of \( L_\mathcal{F} \). 


 Let $x \in C^0(G, F)$ denote an \( nd \)-dimensional vector constructed by column-stacking the individual vectors \( x_v \). This vector \( x \) is then transformed into a vertex feature matrix \( X \in \mathbb{R}^{(nd) \times f} \), where each column corresponds to a vector in \( C^0(G; F) \) and \( f \) is the number of feature channels. The 
 sheaf standard diffusion process is defined by the following Partial Differential Equation (PDE),

\begin{equation}
  \begin{aligned}
X(0) &= X, \hspace{0.5cm}
\frac{\partial X(t)}{\partial t} &= - \Delta_{\mathcal{F}} X(t).
\end{aligned} 
\label{eq:epde}
\end{equation}

Given a diffusion coefficient $\alpha\in[0,1]$, the discrete-time update rule for diffusion is defined as
\begin{equation}
    X(t+1) = X(t) - \alpha\cdot \frac{\partial X(t)}{\partial t}; \alpha > 0.
\end{equation}

When vertex features do not change between time steps, (i.e., \(X(t+1) = X(t)\)), a fixed point \(X(t)\) is reached, which is also referred to as a \emph{steady state}.


\subsection{Sheaf Reaction Diffusion}

Traditional diffusion GNNs including sheaf diffusion in \cref{eq:epde} \cite{bodnar2022neural,hansen2020sheaf} use a uniform diffusion process that primarily focuses on influence spreading but overlooks how individual vertex characteristics and neighboring state transitions shape the process, which is essential for modeling influence diffusion. To address this issue, we model influence dynamics as a diffusion-reaction process 
\cite{turing1990chemical,choi2023gread}. In this process, \emph{diffusion} refers to the spread of information across the network, governed by its connectivity structure, and \emph{reaction} involves altering this information based on interactions and inherent characteristics of vertices within the network. 

\paragraph{Reaction Operators} To capture complex dynamics related to influence diffusion, we introduce the \emph{sheaf reaction diffusion} which extends the sheaf standard diffusion with two reaction operators for \emph{pointwise dynamics} and \emph{coupled dynamics}. That is,
\begin{equation}
\begin{split}
\frac{\partial X_v(t)}{\partial t} &= - \underbracket[1pt]{\alpha \Delta_{\mathcal{F}} X_v(t)}_{\text{Diffusion}} + \underbracket[1pt]{\beta A\left(X_v(t) \right)}_{\text{Pointwise Dynamics}} \\
&\quad + \underbracket[1pt]{\gamma R_v\left(X(t), A, S^t\right)}_{\text{Coupled Dynamics}}
\end{split}
\label{eq:spde}
\end{equation}

Here, \( X_v(t) \) represents the feature vector of vertex \( v \) at time \( t \).
 \(\alpha\), \(\beta\), and \(\gamma\) control the contribution of each term, and \(S^t \in [0, 1]^{n}\) is a vector representing the activation probabilities of vertices at time step \(t\).

 \subparagraph{Pointwise Dynamics} Pointwise dynamics refers to inherent characteristics that govern its activation or susceptibility. For instance, in models of opinion dynamics or epidemic propagation, a vertex may alter its activation state due to internal factors such as personal reassessment or spontaneous recovery, independent of its interactions with other vertices. Thus, we design a reaction operator to capture such vertex evolution as
 \begin{equation*}
     A\left(X_v(t) \right) = \Phi_v^1 \odot \frac{X_v(t)}{\kappa_v^1 + |X_v(t)|}
 \end{equation*}
where $\Phi_v^1, \kappa_v^1 \in \mathbb{R}^{d \times f}$ are coefficient vectors with $\kappa_v^1 > 0, \kappa_v^1 \neq -|X_v(t)|$, and $\odot$ is the Hadamard product.

 \subparagraph{Coupled Dynamics} Coupled dynamics describe how neighboring interactions of vertices influence the activation or susceptibility of a vertex, creating a dynamic interplay between individual behaviors and network-wide interactions. We define a reaction operator for vertex \( v \in V \) as

\begin{align*}
X_v^{\dagger}(t) &=\sum_{u \in N(v)} \bigl(  S_u^t \cdot X_u(t) \bigl) -  \sum_{u \in N(v)}  \bigl( (1 - S_u^t) \cdot X_u(t) \bigl);
\end{align*}

\begin{equation*}
R_v\left(X(t), A, S^t\right) = \Phi_v^2 \odot \frac{X_v^{\dagger}(t)}{\kappa_v^2 + |X_v^{\dagger}(t)|}
\end{equation*}
where \(\Phi_v^2\) and \(\kappa_v^2\) are coefficient vectors with $\kappa_v^2 > 0, \kappa_v^2 \neq -|X_v^{\dagger}(t)|$ for any $t > 0$, and \(S_u^t\) denotes the activation probability of vertex \(u\) at time \(t\). This reaction operator accounts for the combined impact of both activated and susceptible neighbors by producing either positive or negative effects: a high activation level in the neighborhood  yields a positive impact, whereas a high susceptibility leads to a negative impact. Thus, it can enhance the model’s ability to capture complex dynamics in a network.

\begin{remark} Recent research has largely focused on progressive models \cite{chen2022information, chen2022adaptive}, where vertices remain indefinitely active once activated. However, this is often unrealistic, as many real-world propagation models involve vertices transitioning between active and inactive states \cite{lou2014modeling}. While diffusion terms typically follow fixed-point dynamics, our sheaf diffusion model uses reaction operators to introduce non-monotonic, adaptive, and oscillatory behaviors, capturing dynamic vertex state transitions.
\end{remark}

\paragraph{Stability of Reaction Diffusion} Stability in a diffusion process refers to its ability to maintain a controlled and predictable behavior over time \cite{wang2022acmp}. A stable process maintains vertex features within a bounded region over time. This bounded behavior allows the diffusion process to converge to or remain near a fixed point. If vertex features become unbounded, it signals instability, indicating that the diffusion process can deviate significantly from a fixed point or expected behavior \cite{bellman1947boundedness,sastry2013nonlinear}. Such instability can lead to unpredictable outcomes and make it difficult to control or interpret diffusion dynamics effectively.
In the following, we explain the bounded behavior of the reaction operators.

\begin{restatable}[]{lemma}{lemmbound}
The reaction operators in \cref{eq:spde} are bounded for all \( t > 0 \) and \( v \in V \) in terms of norms $||\cdot||$. We have \( \| A(X_v(t)) \| < \| \Phi_v^1 \| \) and \( \| R_v(X(t), A, S^t) \| < \| \Phi_v^2 \| \). 
\label{lemma:bound}
\end{restatable}

The boundedness of the reaction operators ensures that the fixed-point of \cref{eq:spde} is also bounded, as demonstrated in the lemma below.

\begin{restatable}[]{lemma}{lemmstate} 
Let $X^{\star}$ denote the fixed point of \cref{eq:spde}. Given that $\Delta_{\mathcal{F}}$ is well-defined (i.e., positive definite) and operates on a finite-dimensional space, $X^{\star}$ is bounded. 
\label{lemm:state}
\end{restatable}

\subsection{Sheaf GNN Training}

Building on the sheaf Laplacian and reaction operators, our GNN applies the following diffusion propagation rule to update vertex features:
\begin{equation}
\begin{aligned}
X_v(t+1) &= X_v(t) - \alpha \bigg( \Delta_{\mathcal{F}} \left( I_n \otimes W_1^t \right) X_v(t) W_2^t \bigg) \\
&\quad + \beta \bigg(\Phi_v^1 \odot \frac{X_v(t)}{\kappa_v^1 + X_v(t)} \bigg) \\
&\quad + \gamma \bigg( \Phi_v^2 \odot \frac{X_v^{\dagger}(t)}{\kappa_v^2 + |X_v^{\dagger}(t)|} \bigg)
\end{aligned}
\end{equation}

Notably, \(W_1^t \in \mathbb{R}^{d \times d}\), \(W_2^t \in \mathbb{R}^{f \times f}\), and \(\Phi_v^1, \Phi_v^2, \kappa_v^1, \kappa_v^2 \in \mathbb{R}^{d \times f}\) are learnable weight matrices, \(\otimes\) denotes the Kronecker product, and \(I_n\in \mathbb{R}^{n \times n}\) represents an identity matrix. Moreover, \(X(0)\) is obtained by transforming the vertex features of the network through a multi-layer perceptron (MLP), followed by reshaping, resulting in a matrix of dimensions \((nd) \times f\).

After obtaining the updated vertex embeddings in each iteration, we utilize a non-linear neural function \(f_{\eta}(\cdot)\), parameterized by $\eta$, to determine the activation state of each vertex $v \in V$ at the time step $t+1$ as

\begin{equation}
S_v^{t+1} = f_\eta \bigg(X_v(t+1)\bigg); S_v^{t+1} \in [0, 1].
\end{equation}

We use the Mean Square Error (MSE) loss to measure the difference between the predicted activation probabilities \(\hat{S}\) and the ground truth probabilities \(Y \in [0,1]^n\) as:

 



\begin{equation} \label{IEloss}
\mathcal{L}_{train} = || \hat{S} - Y||^2_2 
\end{equation}

\begin{table*}[htbp]
    \centering
    \begin{minipage}
    {0.98\textwidth}
    \centering
    \begin{adjustbox}{width=1.0\textwidth}
    \begin{tabular}{lcccccccccccc|cccccccccccc}
        \toprule
        & \multicolumn{4}{c}{Cora-ML (IC)} & \multicolumn{4}{c}{Network Science (IC)} & \multicolumn{4}{c}{Power Grid (IC)} & \multicolumn{4}{c}{Cora-ML (LT)} & \multicolumn{4}{c}{Network Science (LT)} & \multicolumn{4}{c}{Power Grid (LT)} \\
        \cmidrule(lr){2-5} \cmidrule(lr){6-9} \cmidrule(lr){10-13} \cmidrule(lr){14-17} \cmidrule(lr){18-21} \cmidrule(lr){22-25}
        Methods & 1\% & 5\% & 10\% & 20\% & 1\% & 5\% & 10\% & 20\% & 1\% & 5\% & 10\% & 20\% & 1\% & 5\% & 10\% & 20\% & 1\% & 5\% & 10\% & 20\% & 1\% & 5\% & 10\% & 20\% \\
        \midrule
        IMM & 8.1 & 26.2 & 37.3 & 50.2 & 5.2 & 16.8 & 27.0 & 45.7 & 5.6 & 17.4 & 31.5 & 51.1 & 1.7 & 34.8 & 52.2 & 66.4 & 2.5 & 11.8 & 18.1 & 33.6 & 4.6 & 19.9 & 31.7 & 56.9 \\
        OPIM & 13.4 & 26.9 & 37.4 & 50.9 & 6.4 & 19.4 & 28.9 & 48.6 & 5.7 & 17.7 & 29.7 & 50.1 & 2.3 & 36.9 & 51.2 & 71.5 & 1.6 & 12.0 & 18.1 & 34.1 & 4.4 & 21.6 & 29.4 & 55.5 \\
        SubSIM & 10.1 & 25.7 & 36.8 & 51.1 & 4.8 & 15.4 & 27.9 & 44.8 & 4.6 & 19.2 & 31.7 & 50.2 & 1.7 & 33.6 & 54.7 & 70.1 & 1.8 & 10.4 & 19.2 & 34.1 & 4.5 & 21.1 & 31.2 & 57.4 \\
        \hline
        IMINfECTOR & 9.6 & 26.8 & 37.7 & 50.6 & 5.4 & 17.9 & 27.8 & 47.6 & 5.4 & 18.2 & 31.6 & 50.9 & 2.1 & 33.9 & 51.3 & 70.6 & 2.1 & 11.8 & 18.7 & 34.5 & 4.2 & 21.3 & 31.6 & 56.2 \\
        PIANO & 9.8 & 25.2 & 37.4 & 51.1 & 5.3 & 18.1 & 27.1 & 47.2 & 5.3 & 18.1 & 31.7 & 50.2 & 2.1 & 33.5 & 53.3 & 69.8 & 2.1 & 11.3 & 19.1 & 33.9 & 4.3 & 21.3 & 31.4 & 57.1 \\
        ToupleGDD & 10.6 & 27.5 & 38.5 & 51.5 & 6.3 & 17.8 & 28.3 & 50.5 & 5.4 & 19.3 & 31.6 & 51.3 & 2.3 & 36.2 & 54.5 & 70.9 & 2.8 & 12.4 & 19.8 & 34.6 & 4.8 & 21.9 & 32.6 & 58.1 \\
        DeepIM & 14.1 & 28.1 & 39.6 & 52.4 & 7.8 & 20.9 & 31.5 & 51.2 & 6.3 & 21.0 & 32.5 & 52.4 & \textbf{13.4} & \textbf{69.2} & \textbf{83.5} & 94.1 & \textbf{4.1} & \textbf{16.6} & 26.7 & 41.5 & 6.3 & 24.4 & 46.8 & 71.7 \\
        \hline
        DeepSN & 11.5 & 25.6 & 40.9 & 52.8 & 6.2 & \textbf{22.0} & \textbf{32.0} & \textbf{52.4} & 6.4 & \textbf{24.0} & 36.9 & \textbf{61.0} & 7.4 & 40.7 & 68.2 & \textbf{95.3} & 2.9 & 14.4 & 25.3 & \textbf{52.0} & \textbf{6.3} & \textbf{24.6} & \textbf{47.2} & \textbf{73.2} \\
        DeepSN$_{SP}$ & \textbf{14.2} & \textbf{30.1} & \textbf{42.3} & \textbf{58.4} & \textbf{7.9} & 18.6 & 30.0 & 51.2 & \textbf{7.1} & 22.4 & \textbf{37.4} & 57.4 & 7.9 & 46.4 & 72.8 & 93.8 & 3.6 & 15.5 & \textbf{27.8} & 51.8 & 5.2 & 23.8 & 40.3 & 68.1 \\
        \bottomrule
    \end{tabular}
    \end{adjustbox}
    \label{tab:combined_results}
    \end{minipage}
    \begin{minipage}{0.98\textwidth}
\centering
    \begin{adjustbox}{width=\textwidth}
    \begin{tabular}{lcccccccccccccccccccccccc}
        \toprule
        & \multicolumn{4}{c}{Cora-ML (SIS)} & \multicolumn{4}{c}{Network Science (SIS)} & \multicolumn{4}{c}{Power Grid (SIS)} & \multicolumn{4}{c}{Jazz (SIS)} & \multicolumn{4}{c}{Random (SIS)} & \multicolumn{4}{c}{Digg (SIS)}\\
        \cmidrule(lr){2-5} \cmidrule(lr){6-9} \cmidrule(lr){10-13} \cmidrule(lr){14-17} \cmidrule(lr){18-21} \cmidrule(lr){22-25} 
        Methods & 1\% & 5\% & 10\% & 20\% & 1\% & 5\% & 10\% & 20\% & 1\% & 5\% & 10\% & 20\% & 1\% & 5\% & 10\% & 20\% & 1\% & 5\% & 10\% & 20\% & 1\% & 5\% & 10\% & 20\%\\
        \midrule
        
        IMM & 2.0 & 9.5 & 15.4 & 27.6 & 1.3 & 5.6 & 12.2 & 22.1 & 1.1 & 5.6 & 11.0 & 22.9 & 7.6 & 37.8 & 55.6 & 67.1 & 2.7 & 12.6 & 20.9 & 37.7 & 2.5 & 9.4 & 16.3 & 32.6\\
        OPIM & 2.3 & 9.3 & 16.2 & 27.2 & 1.4 & 5.9 & 13.0 & 22.1 & 1.2 & 5.9 & 11.1 & 22.4 & 8.2 & 35.1 & 56.8 & 68.3 & 2.8 & 12.5 & 20.2 & 36.1 & 2.3 & 9.3 & 16.5 & 32.3\\
        SubSIM & 2.3 & 9.2 & 16.9 & 28.8 & 1.5 & 5.6 & 12.2 & 23.3 & 1.2 & 5.6 & 11.4 & 21.9 & 2.9 & 30.1 & 53.8 & 67.0 & 3.2 & 14.4 & 24.5 & 39.1 & 2.5 & 9.5 & 16.1 & 32.3\\
        \hline
        IMINfECTOR & 2.1 & 9.4 & 16.1 & 27.9 & 1.7 & 5.8 & 12.4 & 22.3 & 1.3 & 5.5 & 10.2 & 23.1 & 8.8 & 35.4 & 54.8 & 66.2 & 2.5 & 12.4 & 20.5 & 36.6 & 2.3 & 9.1 & 16.4 & 32.4\\
        DeepIM & 7.1 & 16.1 & 21.9 & 30.8 & 2.7 & 8.7 & 15.1 & 25.1 & 1.9 & 7.6 & 13.3 & 23.8 & 27.1 & 57.1 & 68.1 & 74.1 & 3.2 & 14.4 & 24.5 & 39.1 & 5.6 & 11.4 & 18.8 & \textbf{36.3} \\
        \hline
        DeepSN & 12.8 & 24.7 & 34.2 & 46.5 & 2.0 & 9.6 & 16.1 & 28.1 & \textbf{2.4} & \textbf{9.8} & \textbf{15.7} & 25.3 & \textbf{34.3} & \textbf{64.9}  & \textbf{75.6}  & \textbf{85.8} & \textbf{5.2} & \textbf{24.2} & \textbf{37.6} & \textbf{54.1}  & \textbf{16.1} & 20.2  & \textbf{24.7} & 33.2\\
        DeepSN$_{SP}$ & \textbf{16.8} & \textbf{29.9} & \textbf{37.9} & \textbf{46.9} &  \textbf{2.7} & \textbf{9.9} & \textbf{17.4} & \textbf{29.0} & 2.1 & 8.2 & 14.9 & \textbf{26.3} & 35.2 &  58.3 & 73.4 & 82.4 & 4.1 & 18.7 & 32.8 & 52.6 & 16.1 & \textbf{20.3} & 23.6 & 33.1  
        \\
        \bottomrule
    \end{tabular}
    \end{adjustbox}
    \label{tab:sis_results}
    
    \end{minipage}
    \caption{Performance of DeepSN variants for influence maximization, compared to baseline methods, under IC, LT, and SIS models. The best results are highlighted in \textbf{bold}. Baseline results are sourced from \citet{ling2023deep}.}
    \label{tab:im_results}
\end{table*}

\section{Optimizing Seed Selection}

In this section, we present a learning-based approach to optimally select a seed set that maximizes influence spread. 
 Selecting an optimal seed set for influence maximization faces a combinatorial explosion, with vertex combinations growing exponentially as the graph size increases.

 We address this challenge leveraging two observations. 
First, instead of relying on the adjacency matrix for graph connectivity, which fails to capture network dynamics in diffusion models, we learn the connectivity through sheaf coefficients, enhancing model flexibility. This leads to a weighted graph $G^w$ that is constructed from sheaf coefficients and the adjacency matrix of the input graph. Then, we employ a partitioning mechanism to identify subgraphs, which reduces the search space in the process of determining the optimal seed set. We apply the Louvain algorithm~\cite{blondel2008fast,dugue2015directed,traag2019louvain} to divide the graph $G^w$ into \( r \) subgraphs  $\{G_i\}_{i=1}^r$ , minimizing the overlap of influence between vertices in different subgraphs. Then, allocating seed vertices across subgraphs can significantly reduce the search space by limiting the number of vertex combinations within each smaller subgraph, rather than across the entire graph.  

We train a neural network $\mathcal{T}_{\phi}$, parameterized by $\phi$, to select seed vertices within subgraphs in a learnable manner. More specifically, $\mathcal{T}_{\phi}$ learns to select $S_i\subseteq V_i$ seed vertices from each subgraph $G_i=(V_i,E_i)$ such that $ S= \bigcup_{i=1}^{k} S_i$ can maximize the overall influence spread 
\begin{equation*}
  S^* = \mathcal{T}_{\phi} \bigg(\{G_i\}_{i=1}^r, G^w\bigg) \text{ subject to }\bigwedge_{i\in [1,r]} |S_i| \leq \frac{k}{n}|V_i|.
\end{equation*}
The condition in the above equation ensures that the number of seeds is chosen  proportionally to the size of each subgraph. $\mathcal{T}_{\phi}$ is trained using a loss function based on the difference between the maximal influence and the predicted influence
\begin{equation}
\begin{aligned}
    \mathcal{L}_{train} &= n - \sigma\bigg(\bigcup_{i=1}^{r} S_i, G^w; \theta\bigg).
\end{aligned}
\end{equation}

 Note that we use the sheaf GNN with trained parameters $\theta$ to approximate the influence diffusion function~$\sigma(\cdot)$.

\section{Complexity Analysis} 


We first analyze the layer-wise complexity of our GNN component. The diffusion operation within our GNN has a complexity of $O\left(n \left(f^2d^2 + d^3\right) + m \left(fd + d^3\right)\right)$, where \( m \) is the number of edges. The complexity of this operation is similar to that reported in \citet{bodnar2022neural}. Each reaction operator introduces an additional complexity of \( O(ndf) \). Consequently, the total complexity of our GNN is
$O\left(n \left(f^2d^2 + d^3\right) + m \left(fd + d^3\right)\right)$. Given that we employ $d= \{1, 2\}$ in our experiments, our GNN incurs only a constant overhead compared to traditional GNNs, such as GCN \cite{kipf2016semi}.

The Louvain algorithm, used as a preprocessing step, has a time complexity of \(O(l m)\), where \(l\) is the number of iterations, and a space complexity of \(O(n + m)\) \cite{lancichinetti2009community}. In our experiments, we employ DeepSN$_{\text{SP}}$ with a sparsified graph structure, resulting in a time complexity of \(O(l\times n\log n)\) and a space complexity of \(O(n)\). We implement $\mathcal{T}_{\phi}$
  using a multi-layer perceptron (MLP) with a time complexity of $O(ndh)$, where $h$ represents the number of hidden neurons in the MLP.

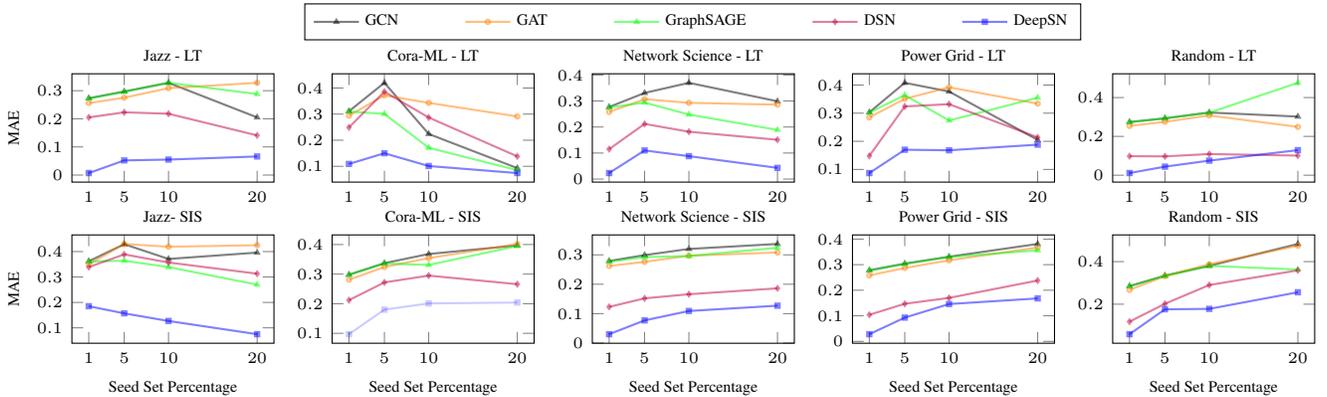
\begin{figure*}[t]
    \centering
\begin{tikzpicture}
    \begin{groupplot}[
        group style={
            group name=my plots,
            group size=5 by 2,
            xlabels at=edge bottom,
            ylabels at=edge left,
            horizontal sep=22pt,
            vertical sep=20pt
        },
        width=0.24\textwidth,
        height=0.17\textwidth,
        xlabel={Seed Set Percentage},
        ylabel={MAE},
        xtick=data,
        xticklabel style={rotate=0, anchor=north},
        legend style={at={(1.15, 1.65)}, anchor=north west, legend columns=6, font=\tiny, /tikz/every even column/.append style={column sep=0.8cm}},
        cycle list name=color list,  
        every axis plot/.append style={thick},
        title style={yshift=-1.0ex, font=\tiny},
        xlabel style={font=\tiny},
        ylabel style={font=\tiny},
        tick label style={font=\tiny}
    ]
    
    \nextgroupplot[title={Jazz - LT}]
    \addplot[black, mark=triangle, mark size=0.8 pt, opacity = 0.5]  table[x=SeedSetSize,y=GCN] {jazz_lt.dat};
    \addlegendentry{GCN}
    \addplot[orange, mark=o, mark size=0.8 pt, opacity = 0.5] table[x=SeedSetSize,y=GAT] {jazz_lt.dat};
    \addlegendentry{GAT}
    \addplot[green, mark=triangle*, mark size=0.8 pt, opacity = 0.5] table[x=SeedSetSize,y=GraphSAGE] {jazz_lt.dat};
    \addlegendentry{GraphSAGE}
    \addplot[purple, mark=diamond*, mark size=0.8 pt, opacity = 0.5] table[x=SeedSetSize,y=DSN] {jazz_lt.dat};
    \addlegendentry{DSN}
    \addplot[blue, mark=square*, mark size=0.8 pt, opacity = 0.5] table[x=SeedSetSize,y=DeepSN] {jazz_lt.dat};
    \addlegendentry{DeepSN}

    \nextgroupplot [title={Cora-ML - LT}]
    \addplot[black, mark=triangle, mark size=0.8 pt, opacity = 0.5] table[x=SeedSetSize,y=GCN] {cora_ml_lt.dat};
    \addplot[orange, mark=o, mark size=0.8 pt, opacity = 0.5] table[x=SeedSetSize,y=GAT] {cora_ml_lt.dat};
    \addplot[green, mark=triangle*, mark size=0.8 pt, opacity = 0.5] table[x=SeedSetSize,y=GraphSAGE] {cora_ml_lt.dat};
    \addplot[purple, mark=diamond*, mark size=0.8 pt, opacity = 0.5] table[x=SeedSetSize,y=DSN] {cora_ml_lt.dat};
    \addplot[blue, mark=square*, mark size=0.8 pt, opacity = 0.5] table[x=SeedSetSize,y=DeepSN] {cora_ml_lt.dat};

    \nextgroupplot[title={Network Science - LT}]
    \addplot[black, mark=triangle, mark size=0.8 pt, opacity = 0.5] table[x=SeedSetSize,y=GCN] {netscience_lt.dat}; 
    \addplot[orange, mark=o, mark size=0.8 pt, opacity = 0.5] table[x=SeedSetSize,y=GAT] {netscience_lt.dat}; 
    \addplot[green, mark=triangle*, mark size=0.8 pt, opacity = 0.5] table[x=SeedSetSize,y=GraphSAGE] {netscience_lt.dat}; 
    \addplot[purple, mark=diamond*, mark size=0.8 pt, opacity = 0.5] table[x=SeedSetSize,y=DSN] {netscience_lt.dat}; 
    \addplot[blue, mark=square*, mark size=0.8 pt, opacity = 0.5] table[x=SeedSetSize,y=DeepSN] {netscience_lt.dat}; 

    \nextgroupplot[title={Power Grid - LT}]
    \addplot[black, mark=triangle, mark size=0.8 pt, opacity = 0.5] table[x=SeedSetSize,y=GCN] {power_grid_lt.dat}; 
    \addplot[orange, mark=o, mark size=0.8 pt, opacity = 0.5] table[x=SeedSetSize,y=GAT] {power_grid_lt.dat}; 
    \addplot[green, mark=triangle*, mark size=0.8 pt, opacity = 0.5] table[x=SeedSetSize,y=GraphSAGE] {power_grid_lt.dat}; 
    \addplot[purple, mark=diamond*, mark size=0.8 pt, opacity = 0.5] table[x=SeedSetSize,y=DSN] {power_grid_lt.dat}; 
    \addplot[blue, mark=square*, mark size=0.8 pt, opacity = 0.5] table[x=SeedSetSize,y=DeepSN] {power_grid_lt.dat}; 

    \nextgroupplot[title={Random - LT}]
    \addplot[black, mark=triangle, mark size=0.8 pt, opacity = 0.5] table[x=SeedSetSize,y=GCN] {random5_lt.dat}; 
    \addplot[orange, mark=o, mark size=0.8 pt, opacity = 0.5] table[x=SeedSetSize,y=GAT] {random5_lt.dat}; 
    \addplot[green, mark=triangle*, mark size=0.8 pt, opacity = 0.5] table[x=SeedSetSize,y=GraphSAGE] {random5_lt.dat}; 
    \addplot[purple, mark=diamond*, mark size=0.8 pt, opacity = 0.5] table[x=SeedSetSize,y=DSN] {random5_lt.dat}; 
    \addplot[blue, mark=square*, mark size=0.8 pt, opacity = 0.5] table[x=SeedSetSize,y=DeepSN] {random5_lt.dat}; 

    \nextgroupplot[title={Jazz- SIS}]
    \addplot[black, mark=triangle, mark size=0.8 pt, opacity = 0.5] table[x=SeedSetSize,y=GCN] {jazz_sis.dat}; 
    \addplot[orange, mark=o, mark size=0.8 pt, opacity = 0.5] table[x=SeedSetSize,y=GAT] {jazz_sis.dat}; 
    \addplot[green, mark=triangle*, mark size=0.8 pt, opacity = 0.5] table[x=SeedSetSize,y=GraphSAGE] {jazz_sis.dat}; 
    \addplot[purple, mark=diamond*, mark size=0.8 pt, opacity = 0.5] table[x=SeedSetSize,y=DSN] {jazz_sis.dat}; 
    \addplot[blue, mark=square*, mark size=0.8 pt, opacity = 0.5] table[x=SeedSetSize,y=DeepSN] {jazz_sis.dat};

    \nextgroupplot[title={Cora-ML - SIS}]
    \addplot[black, mark=triangle, mark size=0.8 pt, opacity = 0.5] table[x=SeedSetSize,y=GCN] {cora_ml_sis.dat}; 
    \addplot[orange, mark=o, mark size=0.8 pt, opacity = 0.5] table[x=SeedSetSize,y=GAT] {cora_ml_sis.dat}; 
    \addplot[green, mark=triangle*, mark size=0.8 pt, opacity = 0.5] table[x=SeedSetSize,y=GraphSAGE] {cora_ml_sis.dat}; 
    \addplot[purple, mark=diamond*, mark size=0.8 pt, opacity = 0.5] table[x=SeedSetSize,y=DSN] {cora_ml_sis.dat}; 
    \addplot[blue, mark=square*, mark size=0.8 pt, opacity = 0.2] table[x=SeedSetSize,y=DeepSN] {cora_ml_sis.dat};

    \nextgroupplot[title={Network Science - SIS}]
    \addplot[black, mark=triangle, mark size=0.8 pt, opacity = 0.5] table[x=SeedSetSize,y=GCN] {netscience_sis.dat}; 
    \addplot[orange, mark=o, mark size=0.8 pt, opacity = 0.5] table[x=SeedSetSize,y=GAT] {netscience_sis.dat}; 
    \addplot[green, mark=triangle*, mark size=0.8 pt, opacity = 0.5] table[x=SeedSetSize,y=GraphSAGE] {netscience_sis.dat}; 
    \addplot[purple, mark=diamond*, mark size=0.8 pt, opacity = 0.5] table[x=SeedSetSize,y=DSN] {netscience_sis.dat}; 
    \addplot[blue, mark=square*, mark size=0.8 pt, opacity = 0.5] table[x=SeedSetSize,y=DeepSN] {netscience_sis.dat};

    \nextgroupplot[title={Power Grid - SIS}]
    \addplot[black, mark=triangle, mark size=0.8 pt, opacity = 0.5] table[x=SeedSetSize,y=GCN] {power_grid_sis.dat}; 
    \addplot[orange, mark=o, mark size=0.8 pt, opacity = 0.5] table[x=SeedSetSize,y=GAT] {power_grid_sis.dat}; 
    \addplot[green, mark=triangle*, mark size=0.8 pt, opacity = 0.5] table[x=SeedSetSize,y=GraphSAGE] {power_grid_sis.dat}; 
    \addplot[purple, mark=diamond*, mark size=0.8 pt, opacity = 0.5] table[x=SeedSetSize,y=DSN] {power_grid_sis.dat}; 
    \addplot[blue, mark=square*, mark size=0.8 pt, opacity = 0.5] table[x=SeedSetSize,y=DeepSN] {power_grid_sis.dat};  

    \nextgroupplot[title={Random - SIS}]
    \addplot[black, mark=triangle, mark size=0.8 pt, opacity = 0.5] table[x=SeedSetSize,y=GCN] {random5_sis.dat}; 
    \addplot[orange, mark=o, mark size=0.8 pt, opacity = 0.5] table[x=SeedSetSize,y=GAT] {random5_sis.dat}; 
    \addplot[green, mark=triangle*, mark size=0.8 pt, opacity = 0.5] table[x=SeedSetSize,y=GraphSAGE] {random5_sis.dat}; 
    \addplot[purple, mark=diamond*, mark size=0.8 pt, opacity = 0.5] table[x=SeedSetSize,y=DSN] {random5_sis.dat}; 
    \addplot[blue, mark=square*, mark size=0.8 pt, opacity = 0.5] table[x=SeedSetSize,y=DeepSN] {random5_sis.dat}; 
    \end{groupplot}
\end{tikzpicture}\vspace{-0.3cm}
\caption{Performance of DeepSN for influence estimation in terms of MAE (Mean Absolute Error), compared to baseline methods. Results under IC model are provided in the appendix.}
\label{fig:estimation}
\end{figure*}

\section{Vertex Feature Separability}\label{sec:expressivity}

We investigate the separation power of our sheaf GNN for vertex features, which plays a crucial role in mitigating oversmoothing ~\cite{bodnar2022neural}. The following proposition shows the existence of the fixed point in the sheaf diffusion process under \cref{eq:epde} and the corresponding properties of transformation maps.

\begin{restatable}[]{proposition}{propsn}
For any graph $G = (V, E)$ with sheaf Laplacian $L_{\mathcal{F}}$ and initial vertex features $X(0)$, there exists a unique fixed-point \(X(t)\) in the sheaf diffusion process 
such that for any $(v,u)\in E$,  \( \mathcal{F}_{v \trianglelefteq e} x_v = \mathcal{F}_{u \trianglelefteq e} x_u \) holds. 
\end{restatable}

Due to the condition $\mathcal{F}_{v \trianglelefteq e} x_v = \mathcal{F}_{u \trianglelefteq e} x_u$ at the fixed point, the sheaf diffusion process has limited capacity to separate distinct vertex features between a vertex \(v\) and its neighbors \(u\), as shown in \cref{propsn-2}. 

 

\begin{restatable}[]{proposition}{propsnt} There exists a graph $G=(V,E)$ with the fixed point \(X(t)\) in the sheaf diffusion process 
which satisfies $x_u = x_v$ for at least one $(u,v) \in E$.
\label{propsn-2}
\end{restatable}


Unlike traditional sheaf diffusion, the sheaf reaction diffusion (\cref{eq:spde}) does not require the condition \( \mathcal{F}_{v \trianglelefteq e} x_v = \mathcal{F}_{u \trianglelefteq e} x_u \) to be satisfied upon convergence. The following  proposition demonstrates this.

 \begin{restatable}[]{proposition} {propep}
For any graph $G = (V, E)$ with sheaf Laplacian $L_{\mathcal{F}}$ and initial vertex features $X(0)$, there exists a unique fixed-point \(X(t)\) in the sheaf reaction diffusion process; 
however, \( \mathcal{F}_{v \trianglelefteq e} x_v = \mathcal{F}_{u \trianglelefteq e} x_u \) does not necessarily hold for each $(v,u)\in E$. 
\label{propep}
\end{restatable}

\begin{remark}
The relaxation of the condition \( \mathcal{F}_{v \trianglelefteq e} x_v = \mathcal{F}_{u \trianglelefteq e} x_u \) in \cref{propep} significantly improves the separability of distinct vertex features. This allows sheaf GNN to learn vertex features that are distinguishable from their neighbors through more powerful transformation maps, offering an advantage over existing sheaf diffusion networks \cite{bodnar2022neural}.
\end{remark}

\section{Experiments}

A set of experiments was conducted to evaluate the performance of DeepSN, considering two variants: \emph{DeepSN} and \emph{DeepSN$_{SP}$}, a computationally efficient variant that uses sheaf coefficients to sparsify the graph structure. We focus on three diffusion models: IC, LT, and SIS \cite{li2018influence}. IC and LT are progressive models, while SIS is non-progressive. 

\paragraph{Datasets} We evaluate DeepSN against other methods using a diverse set of datasets, including five real-world datasets (Jazz \cite{rossi2015network}, Network Science \cite{rossi2015network}, Cora-ML \cite{mccallum2000automating}, Power Grid \cite{rossi2015network}, and Digg \cite{lerman2008analysis}) and one synthetic dataset (Random \cite{ling2023deep}), which range from small graphs to those with over 250,000 vertices.

\paragraph{Experimental Setups and Baselines} We follow the experimental setup of \citet{ling2023deep} for our influence maximization task, comparing DeepSN with traditional baselines: IMM \cite{tang2015influence}, OPIM \cite{tang2018online}, SubSIM \cite{guo2020influence} and learning-based methods: IMINFECTOR \cite{panagopoulos2020multi}, ToupleGDD \cite{chen2023touplegdd}, PIANO \cite{li2022piano}, DeepIM \cite{ling2023deep}.

Additional details on dataset statistics, experimental setups, and model hyperparameters are in the appendix.


\subsubsection{Exp--1. Performance of DeepSN~} We evaluate DeepSN's performance in selecting an optimal seed set for the IM task across various budget constraints $\{1\%,5\%,10\%, 20\%\}$ (i.e. seed set size as a percentage of total number of vertices). The results are demonstrated in  \cref{tab:im_results}. We observe that both DeepSN and DeepSN$_{SP}$ outperform or deliver comparable performance across all diffusion models. Notably, DeepSN achieves a substantial improvement over all baseline methods for the SIS diffusion model. This enhanced performance is attributed to DeepSN’s capability to effectively capture complex diffusion dynamics inherent to non-progressive diffusion models. The experimental results for more datasets are provided in the appendix.

\subsubsection{Exp--2. Ablation Study: Sheaf GNN}
In Figure~\ref{fig:estimation}, we compare the influence estimation performance of DeepSN with several widely used GNNs. The results show that DeepSN significantly outperforms traditional GNNs, such as GCN \cite{kipf2016semi}, GAT \cite{velivckovic2018graph}, and GraphSAGE \cite{hamilton2017inductive}, across all diffusion models and datasets, highlighting its superior effectiveness in influence estimation. These traditional GNNs often struggle to capture long-range dependencies due to oversmoothing, leading to lower performance. Furthermore, while existing sheaf neural networks such as DSN~\cite{bodnar2022neural} address some of these issues, they still fall short in capturing the intricate dynamics required for accurate influence estimation, resulting in suboptimal performance.

\subsubsection{Exp--3. Ablation Study: IM Model}
We evaluate the effectiveness of our IM model $\mathcal{T}_{\phi}$ by replacing it with different IM variants. These variants include: \textbf{DeepSN-CELF}, which incorporates the CELF algorithm \cite{leskovec2007cost} as the IM component; \textbf{DeepSN-WC}, where seed vertices are optimized across the entire network as a single subgraph; and \textbf{DeepSN-WSA}, which uses the adjacency matrix for dividing the graph into subgraphs, instead of sheaf coefficients. The results in Table \ref{tab:im_ablation} show that our IM model outperforms other algorithms. Particularly, DeepSN and DeepSN$_{SP}$ consistently exceed the performance of DeepSN-WC and DeepSN-WSA. This is largely due to  subgraph-based seed optimization and sheaf coefficients.


\begin{table}[ht]
    \centering
    \small 
    \begin{adjustbox}{max width=0.45\textwidth}
    \begin{tabular}{lcccccc}
        \toprule
        & \multicolumn{2}{c}{IC} & \multicolumn{2}{c}{LT} & \multicolumn{2}{c}{SIS} \\
        \cmidrule(lr){2-3} \cmidrule(lr){4-5} \cmidrule(lr){6-7}
        Methods & 10\% & 20\% &  10\% & 20\% &10\% & 20\% \\
        \midrule
        DeepSN-CELF & 37.2 & 52.8  & \textbf{76.0} & 88.6 & 30.6 & 38.7 \\
        DeepSN-WC & 39.2 & 50.4  & 46.3 & 88.6 & 25.3 & 35.9\\
        DeepSN-WSA & 40.0 & 50.6 & 48.0 &  89.6 & 26.0 & 36.0\\
        \hline
        DeepSN & 40.9 & 52.8 & 68.2 & \textbf{95.3} & 34.2 & 46.5 \\
        DeepSN$_{SP}$ & \textbf{42.3} & \textbf{58.4} & 72.8 & 93.8 & \textbf{37.9} & \textbf{46.9} \\
        \bottomrule
    \end{tabular}
    \end{adjustbox}
    \caption{Comparison of DeepSN with various influence maximization approaches for Cora-ML dataset.}
    \label{tab:im_ablation}
\end{table}

\subsubsection{Exp--4. Impact of Layer Depth and Feature Dimension}
We explore the effect of layer depth and feature dimension on the performance of DeepSN. Figure~\ref{fig:depth_and_sd} illustrates how these factors influence the performance of DeepSN in selecting optimal seed vertices for Cora ML dataset, evaluated under the IC diffusion model with a 10\% seed set.

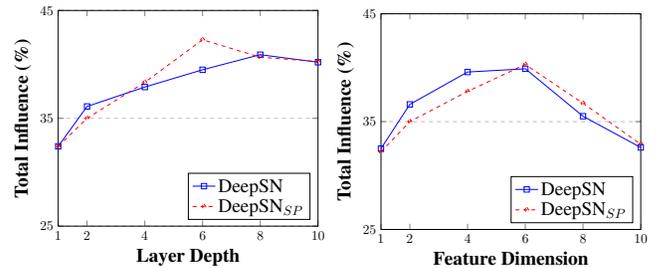
\begin{figure}[ht]
    \centering
    \begin{minipage}[b]{0.49\linewidth}
        \centering
        \resizebox{1.05\linewidth}{!}{
        \begin{tikzpicture}
            \begin{axis}[    
                xlabel={\LARGE Layer Depth},    
                ylabel={\LARGE Total Influence (\%)},
                legend style={font=\LARGE,
                legend cell align=left,
                legend columns=1},
                xlabel style={font=\normalsize},
                ylabel style={font=\normalsize},
                xmin=1, xmax=10,    
                ymin=25, ymax=45,    
                xtick={1, 2,4,6,8,10},    
                ytick={25, 35, 45},    
                legend pos=south east,    
                ymajorgrids=true,    
                grid style=dashed,
                ylabel style={font=\bfseries},
                xlabel style={font=\bfseries}
            ]
            \addplot[    
                color=blue,    
                mark=square,    
                ]
                coordinates {(1, 32.4)(2, 36.1)(4, 37.9)(6, 39.5)(8, 40.9)(10, 40.2)
                };
            \addlegendentry{DeepSN}
            
            \addplot[    
                color=red,    
                mark=diamond, 
                style=dashed,
                ]
                coordinates {
                (1, 32.4)(2, 35.0)(4, 38.34)(6, 42.28)(8, 40.67)(10, 40.26)
                };
            \addlegendentry{DeepSN$_{SP}$}
            \end{axis}
        \end{tikzpicture}
        }
    \end{minipage}\hfill
    \begin{minipage}[b]{0.49\linewidth}
        \centering
        \resizebox{1.05\linewidth}{!}{
                \begin{tikzpicture}
            \begin{axis}[    
                xlabel={\LARGE Feature Dimension},    
                ylabel={\LARGE Total Influence (\%)},
                legend style={font=\LARGE,
                legend cell align=left,
                legend columns=1},
                xmin=1, xmax=10,    
                ymin=25, ymax=45,    
                xtick={1, 2,4,6,8,10},    
                ytick={25,35,45},    
                legend pos=south east,    
                ymajorgrids=true,    
                grid style=dashed,
                ylabel style={font=\bfseries},
                xlabel style={font=\bfseries}
            ]
            \addplot[    
                color=blue,    
                mark=square,    
                ]
                coordinates {(1, 32.5)(2, 36.6)(4, 39.6)(6, 39.9)(8, 35.5)(10, 32.6)
                };
            \addlegendentry{DeepSN}
            
            \addplot[    
                color=red,    
                mark=diamond, 
                style=dashed,
                ]
                coordinates {
                (1, 32.2)(2, 35.0)(4, 37.8)(6, 40.3)(8, 36.7)(10, 32.9)
                };
            \addlegendentry{DeepSN$_{SP}$}
            \end{axis}
        \end{tikzpicture}
        }
    \end{minipage}
    \caption{Impact of layer depth and feature dimension on DeepSN's performance}
    \label{fig:depth_and_sd}
\end{figure}

As expected, DeepSN's performance improves with greater layer depth, effectively capturing long-range interactions and demonstrating robustness against oversmoothing. However, performance only improves with increasing feature dimension up to a certain point. Beyond this, the added complexity, such as more learnable parameters, begins to outweigh the benefits of enhanced representational power, leading to a decline in performance.

In the appendix, we provide an additional ablation study on reaction operators in DeepSN.


\section{Conclusion, Limitations and Future Work}
In this work, we proposed a novel learning framework for the IM problem. Our approach integrates a GNN that harnesses sheaf theory to learn the underlying influence diffusion model in a data-driven manner, while effectively addressing the topological and dynamic complexities of propagation phenomena. Additionally, we proposed a subgraph-based maximization objective to identify the optimal seed set, thereby reducing the combinatorial search space inherent to the IM problem. The empirical results demonstrated the effectiveness of the proposed framework.

Currently, our framework only supports diffusion models with two states. A potential avenue for future work is to extend the framework to handle more complex diffusion models with more than two states, including multi-state threshold models and epidemic models with multiple stages.

\section{Acknowledgments}
This research was supported partially by the Australian Government through the Australian Research Council's Discovery Projects funding scheme (project DP210102273).

\renewcommand{\thesection}{\Alph{section}}
\setcounter{section}{0}
\section*{Appendix}

\begin{table}[ht]
\centering
\caption{Summary of Dataset Statistics}
\begin{tabular}{lccc}
\toprule
\textbf{Dataset}       & \textbf{\# of Vertices}      & \textbf{\# of Edges}       \\
\midrule
Jazz          & 198     & 2,742      \\
Network Science       & 1,565      & 13,532     \\
Cora-ML        & 2,810        &  7,981     \\
Power Grid     & 4,941      & 6,594      \\
Random      & 50,000     & 250,000    \\
Digg           & 279,613    & 1,170,689  \\
\bottomrule
\end{tabular}
\label{tab:dataset_summary}
\end{table}

\section{Experimental Details} \label{sec:exp_det}
\

This section provides details related to our experiments, including dataset statistics, experimental setups, and model hyper-parameters.

\subsection{Datasets} \label{sec:datasets}
We use six datasets in our experiments, comprising five real-world datasets and one synthetic dataset. A summary of these datasets is provided below.
\begin{itemize}
    \item \textit{Jazz \cite{rossi2015network}}: The dataset depicts a social network of jazz musicians, with vertices representing individual musicians and edges denoting collaborations or interactions among them.
    \item \textit{Network Science \cite{rossi2015network}}: This dataset represents a co-authorship network of scientists in the field of network theory. In this dataset, vertices represent individual scientists, and edges indicate collaborations between pairs of scientists.
    \item \textit{Cora-ML \cite{mccallum2000automating}}: The Cora-ML dataset is a citation network in which vertices correspond to scientific papers and edges represent the citation relationships between them.
    \item \textit{Power Grid \cite{rossi2015network}}: The Power Grid dataset represents the configuration of electrical power grids, with vertices indicating power stations or substations and edges representing the transmission lines that link these stations.
    \item \textit{Random \cite{ling2023deep}}: A  synthetic random graph generated using Erd\H{o}s–R\'enyi model \cite{erdds1959random}.
     \item \textit{Digg \cite{lerman2008analysis}}: A dataset sourced from a popular social news website. Each user on the platform is represented as a vertex, with edges indicating user-user relationships (such as friendships or follows) and user-item interactions (such as votes or comments on stories).
\end{itemize}
\cref{tab:dataset_summary} summarizes statistics of these datasets.

\subsection{Experimental Setups}\label{sec:inf_est}

For influence maximization, we adopt the experimental setup outlined by \citet{ling2023deep} for our influence maximization tasks, selecting 1\%, 5\%, 10\%, and 20\% of the vertices in each dataset as seed nodes, simulating each diffusion model until the diffusion process halts, or converges to a steady-state and recording the average influence spread over 100 repetitions. Baseline results are taken from \protect \citet{ling2023deep}.

To evaluate the influence estimation performance in DeepSN, we adhere to the training-testing-validation process outlined by \citet{vabalas2019machine}, splitting the dataset into 60\% for training, 20\% for testing, and 20\% for validation.  For the baseline methods, we use the hyper-parameters reported in their original papers, and report the results.

\subsection{Hyper-parameters} 

In our experiments, we search the hyper-parameters of DeepSN within the following ranges: the number of GNN layers $\in \{2, 5, 10\}$, the dimension $d \in \{1, 2\}$, dropout rate $\in \{0.1, 0.2, 0.5, 0.9\}$, learning rate $\in \{0.001, 0.002, 0.004\}$, batch size $\in \{2, 8, 16, 32\}$, the number of hidden units in the MLP $\in \{32, 64, 128\}$, and the resolution parameter of the Louvain algorithm $\in \{0.1, 1, 2\}$. We employ the Adam algorithm as the optimizer \cite{kingma2014adam}. Further, DeepSN$_{SP}$ employs a threshold of 0.5 to convert continuous sheaf coefficients into binary values. 

\subsection{Computational Resources} \label{sec:computation-resources}

All experiments were performed on a Linux server equipped with an Intel Xeon W-2175 2.50GHz processor with 28 cores, an NVIDIA RTX A6000 GPU, and 512GB of main memory.

\subsection{Additional Experimental Results}

In this section, we provide additional experimental results on the performance of DeepSN.

\subsubsection{Exp--1. Performance of DeepSN}

We present the results of influence maximization performance for the DeepSN models across additional datasets in \cref{tab:combined_results_2}. We observe that both variants of DeepSN either surpass or match the performance of other methods in the influence maximization task for both diffusion models. 

\begin{table*}[htbp]
    \centering
    \begin{adjustbox}{width=\textwidth}
    \begin{tabular}{lcccccccccccc|cccccccccccc}
        \toprule
        & \multicolumn{4}{c}{Jazz (IC)} & \multicolumn{4}{c}{Random (IC)} & \multicolumn{4}{c}{Digg (IC)} & \multicolumn{4}{c}{Jazz (LT)} & \multicolumn{4}{c}{Random (LT)} & \multicolumn{4}{c}{Digg (LT)} \\
        \cmidrule(lr){2-5} \cmidrule(lr){6-9} \cmidrule(lr){10-13}
        \cmidrule(lr){14-17} \cmidrule(lr){18-21} \cmidrule(lr){22-25}
        Methods & 1\% & 5\% & 10\% & 20\% & 1\% & 5\% & 10\% & 20\% & 1\% & 5\% & 10\% & 20\% & 1\% & 5\% & 10\% & 20\% & 1\% & 5\% & 10\% & 20\% & 1\% & 5\% & 10\% & 20\%\\
        \midrule
        IMM & 2.6 & 20.1 & 31.4 & 42.8 & 9.2 & 26.2 & 36.3 & 51.6 & 7.4 & 18.6 & 32.8 & 49.6 & 1.4 & 5.7 & 13.4 & 24.5 & 1.1 & 5.2 & 13.1 & 66.9 & 2.4 & 10.8 & 37.4 & 55.6\\
        OPIM & 2.4 & 20.1 & 34.4 & 46.8 & 9.6 & 25.3 & 36.6 & 51.7 & 7.6 & 18.5 & 32.9 & 48.9 & 1.4 & 6.9 & 12.6 & 20.9 & 1.3 & 5.4 & 12.6 & 62.1 & 2.1 & 11.3 & 38.2 & 57.1 \\
        SubSIM & 3.6 & 18.8 & 37.6 & 44.7 & 9.5 & 26.7 & 36.5 & 51.3 & 7.5 & 18.9 & 33.3 & 49.4 & 1.4 & 5.9 & 11.4 & 21.2 & 1.4 & 5.5 & 13.1 & 69.6 & 2.4 & 11.3 & 37.9 & 56.9\\
        \hline
        IMINfECTOR & 3.6 & 19.7 & 37.5 & 45.9 & 9.1 & 26.2 & 36.1 & 51.3 & 7.9 & 18.6 & 33.5 & 49.8 & 1.4 & 6.2 & 13.5 & 22.8 & 1.3 & 5.2 & 12.9 & 67.4 & 2.2 & 11.1 & 38.9 & 58.7 \\
        PIANO & 2.2 & 19.2 & 36.6 & 43.2 & 9.2 & 25.6 & 36.5 & 51.6 & 7.6 & 18.3 & 33.6 & 49.5 & 1.1 & 6.2 & 12.2 & 22.4 & 1.2 & 5.2 & 12.8 & 67.4 & - & - & - & - \\
        ToupleGDD & 3.3 & 20.4 & 37.2 & 45.7 & 9.5 & 26.8 & 37.1 & 51.4 & - & - & - & - & 1.4 & 6.5 & 12.9 & 23.6 & 1.3 & 5.5 & 13.4 & 70.2 & - & - & - & - \\
        \hline
        DeepIM & 4.9 & 23.3 & 41.5 & 49.9 & \textbf{11.6} & 27.4 & 38.7 & 52.1 & 8.4 & 19.3 & 34.2 & 51.3 & 1.9 & 6.5 & \textbf{16.4} & 99.1 & 1.5 & \textbf{6.5} & \textbf{15.5} & 99.9 & \textbf{3.5} & 15.9 & 41.3 & \textbf{76.2}\\
        \hline
        DeepSN & 8.5 & 26.9 & 41.6 & 53.8 & 10.6 & \textbf{27.8} & \textbf{38.8} & \textbf{52.8} & 8.9 & \textbf{19.5} & 35.2 & \textbf{52.8} & \textbf{2.0} & \textbf{6.7} & 14.9 & 96.9 & \textbf{1.6} & 5.8 & 13.5 & \textbf{99.9} & 3.2 & \textbf{16.1} & \textbf{41.7} & 72.1\\
        DeepSN$_{SP}$ & \textbf{8.8} & \textbf{29.9} & \textbf{41.9} & \textbf{55.6} & 9.8 & 27.2 & 37.6 & 51.6 & \textbf{9.1} & 19.2 & \textbf{35.4} & 51.9 & 1.5 & 5.5 & 14.8 & \textbf{99.5} & 1.3 & 6.2 & 13.3 & 99.9 & 3.3 & 15.9 & 41.2 & 71.6\\
        \bottomrule
    \end{tabular}
    \end{adjustbox}
    \caption{Performance comparison under IC and LT diffusion models. - indicates out-of-memory error. The best results are highlighted in \textbf{bold}.}
    \label{tab:combined_results_2}
\end{table*}

\subsubsection{Exp--2. Ablation Study: Sheaf GNN}

\cref{fig:estimation_ic} demonstrates the performance of Sheaf GNN in estimating influence within the IC diffusion model. Consistent with its superior performance observed in the LT and SIS models, Sheaf GNN continues to outperform all baseline methods.

\subsubsection{Ablation Study: Reaction Operators}

 Table \ref{tab:ablation_study_reaction_terms} shows the impact of different reaction term configurations on DeepSN's performance. The table presents the average total influence(\%) for the Cora-ML dataset across various seed set percentages. While each component offers improvements, the combined approach outperforms all other configurations.

\begin{table}[ht]
    \centering
    \small 
    \begin{adjustbox}{max width=0.45\textwidth}
    \begin{tabular}{lcccc}
        \toprule
        Methods & 1\% & 5\% & 10\% & 20\% \\
        \midrule
        Without Reaction Components & 9.7 & 20.5 & 32.8 & 49.8 \\
        Only Pointwise Dynamics & 10.9 & 23.2 & 34.0 & 51.0 \\
        Only Coupled Dynamics & 10.7 & 22.4 & 37.9 & 50.7 \\
        With Reaction Components & \textbf{11.5} & \textbf{25.6} & \textbf{40.9} & \textbf{52.8} \\
        \bottomrule
    \end{tabular}
    \end{adjustbox}
    \caption{Impact of reaction terms on DeepSN performance. }
    \label{tab:ablation_study_reaction_terms}
\end{table}

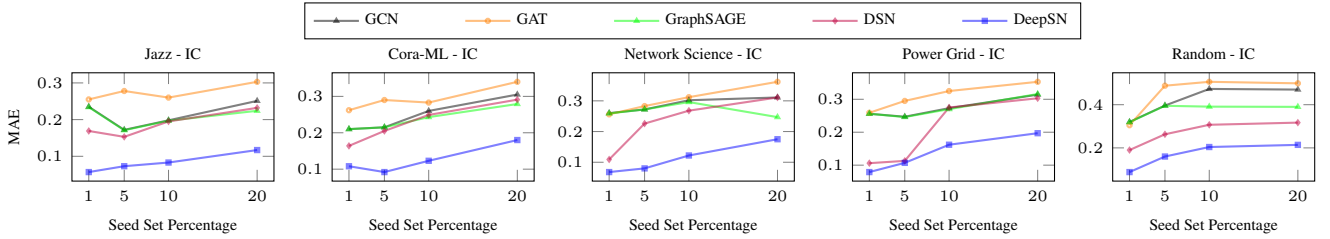
\begin{figure*}[htbp]
    \centering
\begin{tikzpicture}
    \begin{groupplot}[
        group style={
            group name=my plots,
            group size=5 by 1     ,
            xlabels at=edge bottom,
            ylabels at=edge left,
            horizontal sep=22pt,
            vertical sep=20pt
        },
        width=0.24\textwidth,
        height=0.17\textwidth,
        xlabel={Seed Set Percentage},
        ylabel={MAE},
        xtick=data,
        xticklabel style={rotate=0, anchor=north},
        legend style={at={(1.15, 1.65)}, anchor=north west, legend columns=6, font=\tiny, /tikz/every even column/.append style={column sep=0.8cm}},
        cycle list name=color list,  
        every axis plot/.append style={thick},
        title style={yshift=-1.0ex, font=\tiny},
        xlabel style={font=\tiny},
        ylabel style={font=\tiny},
        tick label style={font=\tiny}
    ]
    
    \nextgroupplot[title={Jazz - IC}]
    \addplot[black, mark=triangle, mark size=0.8 pt, opacity = 0.5]  table[x=SeedSetSize,y=GCN] {jazz_ic.dat};
    \addlegendentry{GCN}
    \addplot[orange, mark=o, mark size=0.8 pt, opacity = 0.5] table[x=SeedSetSize,y=GAT] {jazz_ic.dat};
    \addlegendentry{GAT}
    \addplot[green, mark=triangle*, mark size=0.8 pt, opacity = 0.5] table[x=SeedSetSize,y=GraphSAGE] {jazz_ic.dat};
    \addlegendentry{GraphSAGE}
    \addplot[purple, mark=diamond*, mark size=0.8 pt, opacity = 0.5] table[x=SeedSetSize,y=DSN] {jazz_ic.dat};
    \addlegendentry{DSN}
    \addplot[blue, mark=square*, mark size=0.8 pt, opacity = 0.5] table[x=SeedSetSize,y=DeepSN] {jazz_ic.dat};
    \addlegendentry{DeepSN}

    \nextgroupplot [title={Cora-ML - IC}]
    \addplot[black, mark=triangle, mark size=0.8 pt, opacity = 0.5] table[x=SeedSetSize,y=GCN] {cora_ml_ic.dat};
    \addplot[orange, mark=o, mark size=0.8 pt, opacity = 0.5] table[x=SeedSetSize,y=GAT] {cora_ml_ic.dat};
    \addplot[green, mark=triangle*, mark size=0.8 pt, opacity = 0.5] table[x=SeedSetSize,y=GraphSAGE] {cora_ml_ic.dat};
    \addplot[purple, mark=diamond*, mark size=0.8 pt, opacity = 0.5] table[x=SeedSetSize,y=DSN] {cora_ml_ic.dat};
    \addplot[blue, mark=square*, mark size=0.8 pt, opacity = 0.5] table[x=SeedSetSize,y=DeepSN] {cora_ml_ic.dat};

    \nextgroupplot[title={Network Science - IC}]
    \addplot[black, mark=triangle, mark size=0.8 pt, opacity = 0.5] table[x=SeedSetSize,y=GCN] {netscience_ic.dat}; 
    \addplot[orange, mark=o, mark size=0.8 pt, opacity = 0.5] table[x=SeedSetSize,y=GAT] {netscience_ic.dat}; 
    \addplot[green, mark=triangle*, mark size=0.8 pt, opacity = 0.5] table[x=SeedSetSize,y=GraphSAGE] {netscience_ic.dat}; 
    \addplot[purple, mark=diamond*, mark size=0.8 pt, opacity = 0.5] table[x=SeedSetSize,y=DSN] {netscience_ic.dat}; 
    \addplot[blue, mark=square*, mark size=0.8 pt, opacity = 0.5] table[x=SeedSetSize,y=DeepSN] {netscience_ic.dat}; 

    \nextgroupplot[title={Power Grid - IC}]
    \addplot[black, mark=triangle, mark size=0.8 pt, opacity = 0.5] table[x=SeedSetSize,y=GCN] {power_grid_ic.dat}; 
    \addplot[orange, mark=o, mark size=0.8 pt, opacity = 0.5] table[x=SeedSetSize,y=GAT] {power_grid_ic.dat}; 
    \addplot[green, mark=triangle*, mark size=0.8 pt, opacity = 0.5] table[x=SeedSetSize,y=GraphSAGE] {power_grid_ic.dat}; 
    \addplot[purple, mark=diamond*, mark size=0.8 pt, opacity = 0.5] table[x=SeedSetSize,y=DSN] {power_grid_ic.dat}; 
    \addplot[blue, mark=square*, mark size=0.8 pt, opacity = 0.5] table[x=SeedSetSize,y=DeepSN] {power_grid_ic.dat}; 

    \nextgroupplot[title={Random - IC}]
    \addplot[black, mark=triangle, mark size=0.8 pt, opacity = 0.5] table[x=SeedSetSize,y=GCN] {random5_ic.dat}; 
    \addplot[orange, mark=o, mark size=0.8 pt, opacity = 0.5] table[x=SeedSetSize,y=GAT] {random5_ic.dat}; 
    \addplot[green, mark=triangle*, mark size=0.8 pt, opacity = 0.5] table[x=SeedSetSize,y=GraphSAGE] {random5_ic.dat}; 
    \addplot[purple, mark=diamond*, mark size=0.8 pt, opacity = 0.5] table[x=SeedSetSize,y=DSN] {random5_ic.dat}; 
    \addplot[blue, mark=square*, mark size=0.8 pt, opacity = 0.5] table[x=SeedSetSize,y=DeepSN] {random5_ic.dat}; 
    \end{groupplot}
\end{tikzpicture}
\caption{Performance of DeepSN for influence estimation  under IC model.}
\label{fig:estimation_ic}
\end{figure*}

\section{Proofs}

In the following, we provide the proofs of the lemmas and theorems presented in the main content.

\lemmps*
\begin{proof}
Assume \( \hat{L}_\mathcal{F} = L_\mathcal{F} + \epsilon I \) is positive definite. This implies all eigenvalues \( \mu_i = \lambda_i + \epsilon > 0 \), where \( \lambda_i \) are eigenvalues of \( L_\mathcal{F} \). The smallest eigenvalue, \( \lambda_{\text{min}} \), sets the strictest condition: \( \lambda_{\text{min}} + \epsilon > 0 \), or \( \epsilon > -\lambda_{\text{min}} \). Conversely, if \( \epsilon > -\lambda_{\text{min}} \), then \( \mu_i = \lambda_i + \epsilon > 0 \) for all eigenvalues \( \lambda_i \) of \( L_\mathcal{F} \). Therefore, \( \hat{L}_\mathcal{F} \) is positive definite. The proof is complete.

\end{proof}

\lemmbound*
\begin{proof}
Let \( X_v(t) = (X^1_v(t), \dots, X^i_v(t), \dots, X^f_v(t)) \). Then, \( A_v(X(t)) \) can be expressed as:
\[
A_v(X(t)) = \left( \begin{array}{c}
\frac{\Phi_{v,1}^1 X^1_v(t)}{\kappa_{v,1}^1 + |X^1_v(t)|}, \\
\vdots \\
\frac{\Phi_{v,i}^1 X^i_v(t)}{\kappa_{v,i}^1 + |X^i_v(t)|}, \\
\vdots \\
\frac{\Phi_{v,n}^1 X^n_v(t)}{\kappa_{v,n}^1 + |X^n_v(t)|}
\end{array} \right)
\]
As \( X^i_v(t) \rightarrow 0 \):
\[ \lim_{X^i_v(t) \to 0} \frac{\Phi_{v,i}^1 X^i_v(t)}{\kappa_{v,i}^1 + |X^i_v(t)|} = \frac{\Phi_{v,i}^1 \cdot 0}{\kappa_{v,i}^1 + 0} = 0 \]

As \( X^i_v(t) \rightarrow \infty \):
\[ \lim_{X^i_v(t) \to \infty} \frac{\Phi_{v,i}^1 X^i_v(t)}{\kappa_{v,i}^1 + |X^i_v(t)|} = \frac{\Phi_{v,i}^1}{1 + \frac{\kappa_{v,i}^1}{X^i_v(t)}} = \Phi_{v,i}^1 \]

As \( X^i_v(t) \rightarrow -\infty \):
\[ \lim_{X^i_v(t) \to -\infty} \frac{\Phi_{v,i}^1 X^i_v(t)}{\kappa_{v,i}^1 + |X^i_v(t)|} = \frac{\Phi_{v,i}^1}{-1 + \frac{\kappa_{v,i}^1}{X^i_v(t)}} = -\Phi_{v,i}^1 \]

Further, $\frac{ \Phi_{v,i}^1 X^i_v(t)}{\kappa_{v,i}^1 + |X^i_v(t)|} < \Phi_{v,i}^1$ always holds since $\kappa_{v,i}^1 > 0$ and $\kappa_{v,i}^1 \neq -|X^i_v(t)|$. Therefore, the values of \( A_v(X(t)) \) are bounded by the corresponding values of \( \Phi_v^1 \). Applying the same reasoning, we can prove that the values of \( R_v\left(X(t), A, S^t\right) \) are bounded by the corresponding values of \( \Phi_v^2 \). Consequently, the norms of \( A_v(X(t)) \) and \( R_v\left(X(t), A, S^t\right) \) are also bounded by the norms of \( \Phi_v^1 \) and \( \Phi_v^2 \), respectively. This completes the proof.

\end{proof}

\lemmstate*

\begin{proof}
Let us consider the diffusion PDE:

\[
\frac{\partial X(t)}{\partial t} = - \alpha \Delta_{\mathcal{F}} X(t) + \beta A\left(X(t) \right) + \gamma R\left(X(t), A, S^t\right)
\]

Assuming the system reaches a fixed point, we have \( \frac{\partial X(t)}{\partial t} = 0 \) . At this point, the equation simplifies to:

\[
\alpha \Delta_{\mathcal{F}} X^{\dagger} = \beta A\left(X^{\dagger} \right) + \gamma R\left(X^{\dagger}, A, S^t\right)
\]

According to Lemma \ref{lemma:bound}, the terms \( A(X(t)) \) and \( R(X(t), A, S^t) \) are bounded between \( -\Phi_1 \) and \( \Phi_1 \), and \( -\Phi_2 \) and \( \Phi_2 \), respectively, where \( \Phi_1 \) and \( \Phi_2 \) are matrices constructed by stacking the vectors \( \Phi_v^1 \) and \( \Phi_v^2 \) as rows. 

First, we derive an upper bound for the L2 norm of $X^\dagger$.

\[
\|\alpha \Delta_{\mathcal{F}} X^\dagger\|_2 \leq |\beta| \cdot \|\Phi_1\|_2 + |\gamma| \cdot \|\Phi_2\|_2
\]

Rearranging by factoring out \( |\alpha| \) gives:

\[
|\alpha| \|\Delta_{\mathcal{F}} X^\dagger\|_2 \leq |\beta| \cdot \|\Phi_1\|_2 + |\gamma| \cdot \|\Phi_2\|_2
\]

Dividing through by \( |\alpha| \) (assuming \( \alpha \neq 0 \)) leads to:

\[
\|\Delta_{\mathcal{F}} X^\dagger\|_2 \leq \frac{|\beta| \cdot \|\Phi_1\|_2 + |\gamma| \cdot \|\Phi_2\|_2}{|\alpha|}.
\]

Given that \( \Delta_{\mathcal{F}} \) is positive definite, it is symmetric and invertible, with all eigenvalues strictly positive. The largest eigenvalue is denoted as \( \lambda_{\max} \), implying that the \( ||\Delta_{\mathcal{F}}||_2 \) equal to \( \lambda_{\max} \). Combining these facts, the norm of \( X^\dagger \) can be bounded by:

\[
\|X^\dagger\|_2 \leq \frac{1}{|\alpha| \lambda_{\max}} (|\beta| \cdot \|\Phi_1\|_2 + |\gamma| \cdot \|\Phi_2\|_2).
\]

\end{proof}

\propsn* 

\begin{proof}

Let us consider the diffusion PDE given in Eq. \eqref{eq:epde}:
\begin{equation*}
  \begin{aligned}
X(0) = X,\ \frac{\partial X(t)}{\partial t} = - \Delta_{\mathcal{F}} X(t)
\end{aligned} 
\label{eq:pde}
\end{equation*}

 In accordance with Theorem 2.2 in \protect\citet{hansen2021opinion},  \( X \) converges to 
 \(H_0(G; \mathcal{F})\) in the fixed point, where \(H_0(G; \mathcal{F}) = \{x \in C_0(G; \mathcal{F}) \mid \mathcal{F}_{u \trianglelefteq e} x_u = \mathcal{F}_{v \trianglelefteq e} x_v \}\). 
This implies that, for any adjacent vertices \( v \) and \( u \), the following holds:
\[
\mathcal{F}_{v \trianglelefteq e} x_v = \mathcal{F}_{u \trianglelefteq e} x_u.
\]
\end{proof}

\propsnt*
\begin{proof}
We employ Proposition 9 from \citet{bodnar2022neural} to prove this proposition. Consider a set of connected bipartite graphs \( G = (A, B, E) \), with partitions \( A \) and \( B \) forming two distinct classes, where \( |A| = |B| \). According to their proposition, restriction maps defined by symmetric sheaves of dimension 1, specifically

\[
\mathcal{H}_{\text{sym}} := \{ (F, G) : F_{v \trianglelefteq e} = F_{u \trianglelefteq e}, \, \det(F_{v \trianglelefteq e}) \neq 0 \},
\]

\noindent are incapable of separating the classes of any graph in \( G \) for any initial conditions \( X(0) \). This is due to the fact that, for adjacent vertices \( u \) and \( v \), we have \( x_v = x_u \), which prevents differentiation between the classes.

\end{proof}

\propep*

\begin{proof}

We start with the reaction diffusion equation:

\begin{equation*}
    X(t+1) = X(t) -  \frac{\partial X(t)}{\partial t}
\end{equation*}

where, 

\begin{equation*}
 \frac{\partial X(t)}{\partial t} = - \alpha \Delta_{\mathcal{F}} X(t) + \beta A\left(X(t) \right) 
 + \gamma R\left(X(t), A, S^t\right)
\end{equation*}

When the system reaches a fixed point, we have \( \frac{\partial X_(t)}{\partial t} = 0 \). At this point, the equation simplifies to:

\[
\alpha \Delta_{\mathcal{F}} X^{\dagger} = \beta A\left(X^{\dagger} \right) + \gamma R\left(X^{\dagger}, A, S^t\right)
\]

We can write this as:

\[
\alpha \Delta_{\mathcal{F}} X^{\dagger} = \frac{a \cdot X^{\dagger}}{b + X^{\dagger}}  + \frac{c \cdot X^{\dagger}}{d + X^{\dagger}}
\]

\noindent where $a,b,c,d$ are matrices with same dimension as $\Delta_{\mathcal{F}}$. First, we analyze the existence of a non-trivial fixed point. To do this, we can simplify the equation by dividing both sides by \(X^{\dagger}\), assuming \(X^{\dagger} \neq 0\). This gives us the equation \(\alpha \Delta_{\mathcal{F}} = f(X^{\dagger})\), where
\[
f(X^{\dagger}) = \frac{a}{b + X^{\dagger}} + \frac{c}{d + X^{\dagger}}.
\]

The function \(f(X^{\dagger})\) is continuous and strictly decreasing. This is because as \(X^{\dagger}\) increases, the terms \(\frac{a}{b + X^{\dagger}}\) and \(\frac{c}{d + X^{\dagger}}\) both decrease, making the entire function \(f(X^{\dagger})\) decrease. Specifically, when \(X^{\dagger} = 0\), the function \(f(X^{\dagger})\) takes its maximum value, \(\frac{a}{b} + \frac{c}{d}\). As \(X^{\dagger}\) increases towards infinity, the function \(f(X^{\dagger})\) decreases toward \(0\). For a non-trivial fixed point \(X^{\dagger} > 0\) to exist, the value \(\alpha \Delta_{\mathcal{F}}\) must fall within the range of values that \(f(X^{\dagger})\) can take. This means \(\alpha \Delta_{\mathcal{F}}\) must satisfy \(0 < \alpha \Delta_{\mathcal{F}} < \frac{a}{b} + \frac{c}{d}\). The uniqueness of \(X^{\dagger}\) is guaranteed because \(f(X^{\dagger})\) is strictly decreasing and continuous. Therefore, under the condition that \(0 < \alpha \Delta_{\mathcal{F}} < \frac{a}{b} + \frac{c}{d}\), there is a unique, non-trivial fixed point \(X^{\dagger} > 0\) that solves the equation.

We again consider the equation:

\[
\alpha \Delta_{\mathcal{F}} X^{\dagger} = \beta A\left(X^{\dagger} \right) + \gamma R\left(X^{\dagger}, A, S^t\right)
\]

\[
\alpha \Delta_{\mathcal{F}} X^{\dagger} = \frac{a \cdot X^{\dagger}}{b + X^{\dagger}}  + \frac{c \cdot X^{\dagger}}{d + X^{\dagger}}
\]

Since $X^{\dagger}$ can be non-trivial (i.e., not equal to 0), we can see that R.H.S. can be non-zero, leading to $\alpha \Delta_{\mathcal{F}} X^{\dagger} \neq 0$.  This implies that in the fixed-point, the solution does not necessarily converge to \(H_0(G; \mathcal{F}) = \{x \in C_0(G; \mathcal{F}) \mid \mathcal{F}_{u \trianglelefteq e} x_u = \mathcal{F}_{v \trianglelefteq e} x_v \}\) (i.e., it can converge to $\mathcal{F}_{u \trianglelefteq e} x_u \neq \mathcal{F}_{v \trianglelefteq e} x_v$ for adjacent nodes $u$ and $v$). 

\end{proof}

\bibliography{aaai25}

\end{document}